\documentclass{article}

\PassOptionsToPackage{numbers,sort&compress}{natbib}

\usepackage[preprint]{neurips_2026}

\usepackage[T1]{fontenc}
\usepackage[utf8]{inputenc}
\usepackage{microtype}

\usepackage{amsmath,amssymb,amsthm}
\usepackage{mathtools}
\usepackage{graphicx}
\usepackage{booktabs}
\usepackage{longtable}
\usepackage{tabularx}
\usepackage{array}
\usepackage{enumitem}
\usepackage{caption}
\usepackage{authblk}
\usepackage{url}
\usepackage[colorlinks=true,linkcolor=black,citecolor=black,urlcolor=blue]{hyperref}

\newcolumntype{L}[1]{>{\raggedright\arraybackslash}p{#1}}

\theoremstyle{definition}
\newtheorem{innerdefn}{Definition}
\newtheorem{innerthm}{Theorem}
\newtheorem{innerlem}{Lemma}
\newtheorem{innerprop}{Proposition}
\newtheorem{innercor}{Corollary}
\newtheorem{innerrem}{Remark}
\newtheorem{innerassu}{Assumption}

\newenvironment{defn}[2]{\renewcommand{\theinnerdefn}{#1}\innerdefn[#2]}{\endinnerdefn}
\newenvironment{thm}[2]{\renewcommand{\theinnerthm}{#1}\innerthm[#2]}{\endinnerthm}
\newenvironment{lem}[2]{\renewcommand{\theinnerlem}{#1}\innerlem[#2]}{\endinnerlem}
\newenvironment{prop}[2]{\renewcommand{\theinnerprop}{#1}\innerprop[#2]}{\endinnerprop}
\newenvironment{cor}[2]{\renewcommand{\theinnercor}{#1}\innercor[#2]}{\endinnercor}
\newenvironment{rem}[2]{\renewcommand{\theinnerrem}{#1}\innerrem[#2]}{\endinnerrem}
\newenvironment{assu}[2]{\renewcommand{\theinnerassu}{#1}\innerassu[#2]}{\endinnerassu}

\newcommand{\Prob}{\mathcal{P}}                 
\newcommand{\Borel}{\mathcal{B}}                
\newcommand{\Shf}{\mathcal{F}}                  
\newcommand{\Ctx}{\mathcal{C}}                  
\newcommand{\Sys}{\mathfrak{S}}                 
\newcommand{\Ex}{\mathbb{E}}
\newcommand{\Reals}{\mathbb{R}}
\newcommand{\Hist}{H}                           
\newcommand{\Homeo}{\operatorname{Homeo}}
\newcommand{\Fix}{\operatorname{Fix}}
\newcommand{\Hol}{\operatorname{Hol}}
\newcommand{\TV}{\mathrm{TV}}
\newcommand{\Eeq}{E_{\mathrm{eq}}}
\newcommand{\Edesc}{E_{\mathrm{desc}}}
\newcommand{\Escale}{E_{\mathrm{scale}}}
\newcommand{\Esat}{E_{\mathrm{sat}}}
\newcommand{\Dpred}{D_{\mathrm{pred}}}

\allowdisplaybreaks

\begin{document}

\title{Space as an Interventional Invariant:
Cross-Modal Predictive Geometry for Stratified Cities
and Em-Spaced Intelligence}

\author[1]{Tao Yang\thanks{Corresponding author: \texttt{yangtao128@tsinghua.edu.cn}}}
\author[2]{Xuhui Lin}
\author[3]{Kunyao Li}
\author[3]{Haijiang Li}

\affil[1]{School of Architecture, Tsinghua University,
Beijing, China}

\affil[2]{Department of Geography, University College London,
London, United Kingdom}

\affil[3]{School of Engineering, Cardiff University,
Cardiff, United Kingdom}

\date{}

\maketitle

\begin{abstract}
Space is a foundational concept across mathematics, physics, spatial cognition, urban science, and embodied intelligence, yet these fields often treat spatial structure either as a shared geometric container or as a collection of disconnected representations. Such approaches struggle to explain how heterogeneous sensory and urban processes can jointly reveal a common spatial structure, particularly when different modalities do not share the same metric or representation. This paper addresses this gap by defining space as an interventional invariant: the minimal relational structure that preserves local compatibility and the conditional laws of future observations under admissible actions. We develop a cross-modal predictive geometry that integrates local state spaces, modality-specific observation maps, an action groupoid, and a canonical predictive-state quotient, with explicit causal conditions for identifying interventional rather than merely observational structure. The key theoretical result shows that, under joint point separation, equivariance, and interventional faithfulness, the latent space is identifiable up to the centraliser of the intervention group, thereby reducing representational ambiguity to residual coordinate freedom. The framework is further extended to stratified urban systems using sheaf-valued representations, allowing geometric, physical, mobility, social, and economic layers to coexist without being reduced to a single metric. Synthetic experiments under noise evaluate equivariance, predictive sufficiency, holonomy, restriction-map recovery, cross-scale consistency, and context saturation. The resulting framework provides a unified and falsifiable foundation for spatial cognition, urban science, embodied AI, and em-spaced intelligence.
\end{abstract}


\bigskip
\begin{quote}\small\itshape
Space is not identified with a coordinate container. It is the least relational
structure that preserves local compatibility and the intervention-conditioned
future laws shared by heterogeneous modes of observation.
\end{quote}

\section{Introduction}

The word \emph{space} performs several jobs at once. In mathematics it names an
object equipped with specified structure and morphisms. In physics it refers to an
empirically constrained geometry within a dynamical theory. In perception it names
the organisation by which an agent anticipates what will change when it moves.
Urban practice adds a further complication. Buildings and roads are explicit
geometric objects, whereas wind, heat, electromagnetic propagation, sound, traffic,
institutions and exchange generate less visible but equally consequential
geometries of reachability, resistance and flow. A satisfactory theory must relate
these senses without erasing their differences.

The starting observation is that sensory modalities need not resemble one another
in order to disclose a common world. A turn of the head changes retinal flow,
binaural delay and proprioceptive state according to one movement. Opening a door
modifies light, air, heat, sound and accessibility together. Closing a road leaves
the Euclidean map nearly unchanged but transforms travel-time, economic and social
reachability. The stable object is therefore not a shared signal format. It is a
family of transformations that remains jointly predictable under intervention.

This paper proposes that the primary definition of space should be relational and
interventional, while dimension should be treated as a derived complexity index.
Dimension alone cannot determine which states are adjacent, which transitions are
admissible, which local descriptions glue, or which interventions distinguish two
apparently identical situations. Conversely, a relation becomes spatial only when
it supports locality, reachability, composable change, cross-modal covariance and
stable prediction. Mere statistical association is not enough.

\paragraph{Contribution 1.} A typed definition of cross-modal predictive geometry
is given in terms of probes, local states, interventions and conditional future
laws, together with the causal assumptions under which the conditional laws are
interventional rather than merely observational.

\paragraph{Contribution 2.} The central identifiability result is proved: when the
observation family jointly separates points and the learned representation is
equivariant, the latent space is recovered not up to an arbitrary homeomorphism but
up to the centraliser of the intervention group. When the action is simply
transitive this centraliser is the group itself, so space is recovered as a torsor
and coordinates are exactly the residual gauge freedom --- no more and no less.
Removing the interventions collapses the statement to the known impossibility
results for unsupervised disentanglement.

\paragraph{Contribution 3.} Space is shown to be well defined independently of any
single observational context. The context-indexed predictive quotients form a
projective system, and space is its limit; the limit is attained at a finite stage
exactly when a sufficient context exists, which is an experimentally testable
saturation claim.

\paragraph{Contribution 4.} A stratified, fibred and sheaf-valued urban model places
physical and social geometries on a common base without assuming that they share
one distance function. A holonomy obstruction is exhibited that no graph Laplacian
with identity restriction maps can detect, which is what the sheaf formalism buys
over a multilayer network.

\paragraph{Contribution 5.} The formulae are subjected to type, limit and numerical
checks under noise, ablation and refinement, and the construction is translated into
an experimentally testable architecture for embodied and em-spaced intelligence.

\section{Three senses of space}

\subsection{Mathematical space}

A mathematical space is an object $X$ in a category $\mathbf{C}$, together with the
additional structure that makes the intended questions meaningful. A topological
space privileges continuity; a Riemannian manifold adds a metric tensor; a
metric-measure space joins distance to mass; a graph records adjacency; a sheaf
records the passage from local data to compatible global data; a Hilbert space
records linear and inner-product structure. There is no structure-free mathematical
meaning of space. Its identity is always relative to a chosen class of
structure-preserving maps \citep{maclane1992sheaves,gromov2007metric,sturm2024metric}.

\subsection{Physical space}

Physical space is not merely a manifold written in coordinates. It is a model
together with fields, laws, symmetries, measuring operations and error. In general
relativity a three-dimensional space is usually obtained from a Lorentzian spacetime
only after a choice of foliation; in continuum mechanics, the material and spatial
descriptions are related but distinct; in thermodynamics and field theory, geometry
is inferred through what instruments and bodies do. Coordinate changes are
representational. Causal and metrical invariants are empirical
\citep{oneill1983semiriemannian,pearl2009causality}.

\subsection{Perceptual space}

Perceptual space is the organisation of possible sensorimotor consequences. The
sensorimotor tradition correctly insists that the spatial content of seeing, hearing
or touching lies partly in lawful changes under movement
\citep{gibson1979ecological,oregan2001sensorimotor,terekhov2016space,laflaquiere2018discovering}.
Cognitive-map research adds that a useful spatial representation supports flexible
inference beyond immediate sensation \citep{behrens2018cognitive}. The present
proposal sharpens these claims: two embodied histories occupy the same perceptual
state precisely when every admissible future policy induces the same law of future
multimodal observations.

\begin{figure}[htbp]
\centering
\includegraphics[width=\textwidth]{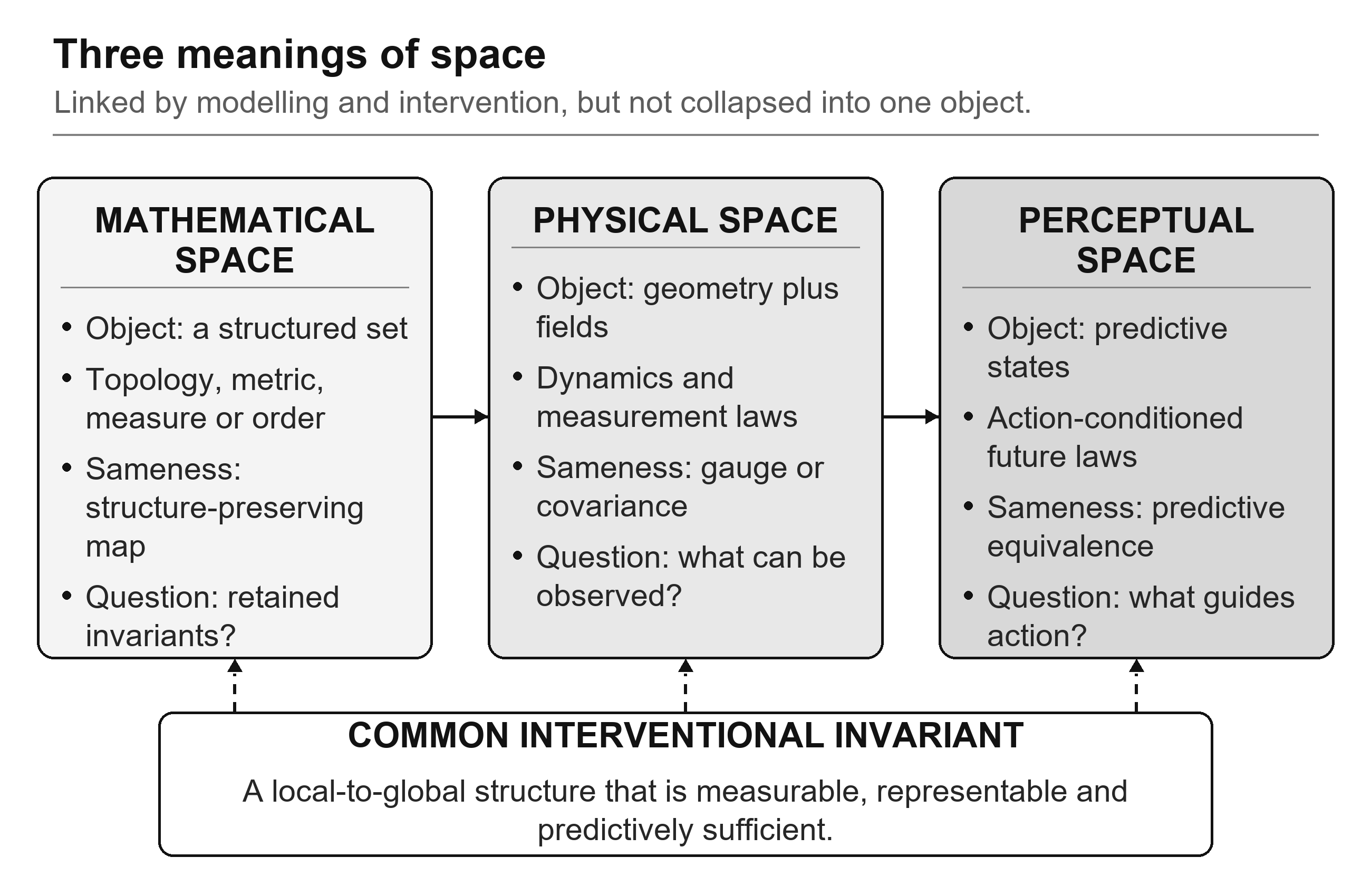}
\caption{Mathematical, physical and perceptual space are linked by modelling,
measurement and action, but their equivalence relations remain distinct.}
\label{fig:three-senses}
\end{figure}

\begin{defn}{2.1}{interventional invariant}
An interventional invariant is an equivalence class of relational models whose
observable conditional laws are unchanged under an admissible change of
representation, but vary lawfully under interventions on the represented system.
\end{defn}

\subsection{Relation to existing predictive and causal formalisms}
\label{sec:related}

The construction below inherits from four literatures, and it is worth stating
precisely what is taken and what is added. From computational mechanics comes the
\emph{causal state}: the equivalence class of pasts inducing the same conditional
distribution over futures, together with its minimality among sufficient statistics
\citep{shalizi2001computational}. From reinforcement learning come predictive state
representations \citep{littman2002predictive}, observable operator models
\citep{jaeger2000observable} and, in the controlled case, stochastic bisimulation
and model minimisation for Markov decision processes \citep{givan2003equivalence},
with their quantitative refinements as bisimulation metrics
\citep{ferns2004metrics,castro2020scalable} and their use as
representation-learning objectives \citep{zhang2021learning}. Definition~3.1 and
Theorem~3.2 are the extension of these notions to a groupoid of interventions acting
on a sheaf of typed local states; the extension is what allows locality, cross-modal
compatibility and scale to constrain the same quotient, but the minimality argument
itself is not new and is not claimed as the contribution.

From causal inference comes the distinction between conditioning and intervening
\citep{pearl2009causality}, the identification of interventional laws from logged
data under sequential ignorability and positivity \citep{robins1986new}, and the
recent programme of causal representation learning \citep{scholkopf2021toward}. The
specific results that make the present paper's central claim provable are the
identifiability theorems for latent variables under interventions
\citep{brehmer2022weakly,ahuja2023interventional,vonkugelgen2023nonparametric,lippe2022citris,squires2023linear,vonkugelgen2021self};
Theorem~3.7 is their topological counterpart, with the residual ambiguity expressed
as a centraliser rather than as a permutation-and-scaling group. From applied
topology come cellular sheaves and their Laplacians
\citep{robinson2017sheaves,hansen2019spectral,bodnar2022neural,hansen2021opinion},
the connection Laplacian and vector diffusion maps
\citep{singer2012vector,bandeira2013cheeger}, and sheaf-based sensor fusion with
non-identity restriction maps \citep{joslyn2020sheaf}.

Two further connections are worth flagging even though they are not developed here.
The successor representation and successor features
\citep{dayan1993improving,barreto2017successor} are the linear, discounted special
case of the predictive state, and the hippocampal predictive-map hypothesis
\citep{stachenfeld2017hippocampus} is the corresponding neuroscientific claim; the
cognitive-map literature \citep{behrens2018cognitive} can therefore be read as an
empirical instance of Definition~3.1 rather than as a separate tradition. Modern
self-supervised world models \citep{lecun2022path} optimise a predictive objective
of the same type, but typically without an explicit action groupoid or descent
constraint, which is precisely the structure the penalties of Section~\ref{sec:learning}
supply.

\section{Cross-modal predictive geometry}

\subsection{Typed observational system}

Let $B$ be a site of probes. A probe may be a spatial region, sensor footprint,
body-centred receptive field, cell of a complex, or finite experimental context. Let
$\Shf$ be a sheaf (or, where homotopy is essential, an $\infty$-sheaf) of local
system states on $B$. For each modality $\alpha$ in a finite index set $A$, let
$Y_\alpha$ be a standard Borel observation space and let
$h_{\alpha,U}\colon \Shf(U)\to Y_\alpha(U)$ be a measurable local observation map.
Let $G$ be a measurable action groupoid: its arrows include bodily motions,
environmental controls and institutional interventions, and composition represents
executable succession.
\begin{equation}
\Sys \;=\; \bigl(B,\ \Shf,\ G,\ \{Y_\alpha, h_\alpha\}_{\alpha\in A},\ \Phi,\ P\bigr).
\label{eq:system}
\end{equation}
Here $\Phi_g$ maps a state on the source of an action $g$ to a state on its target,
while $P$ is the family of regular conditional laws for future observations.
Equation~\eqref{eq:system} is a type declaration, not a claim that all modalities
inhabit one vector space. Their commonality lies in their response to the same
arrows of $G$.
\begin{equation}
h_\alpha\circ\Phi_g\colon \Shf(s(g))\to Y_\alpha(t(g)),
\qquad
T^g_\alpha\circ h_\alpha\colon \Shf(s(g))\to Y_\alpha(t(g)).
\label{eq:equivariance-types}
\end{equation}
The two sides of \eqref{eq:equivariance-types} are now comparable. Approximate
equivariance is measured rather than asserted by an untyped resemblance sign:
\begin{equation}
\Eeq \;=\; \Ex_{g,x}\Bigl[\ \sum_{\alpha} w_\alpha\,
d_\alpha\bigl(h_\alpha(\Phi_g x),\, T^g_\alpha(h_\alpha(x))\bigr)^2\ \Bigr].
\label{eq:eeq}
\end{equation}

\subsection{The canonical predictive state}

Let $\Hist$ be the standard Borel space of finite multimodal histories. Let $\Pi$ be
a countable policy class and $K$ a countable set of horizons. For $\pi\in\Pi$ and
$k\in K$, write $P^{\pi,k}(\cdot\mid h)$ for the regular conditional law of the next
$k$ observations under policy $\pi$. Define
\begin{equation}
S(h) \;=\; \bigl(P^{\pi,k}(\cdot\mid h)\bigr)_{\pi\in\Pi,\,k\in K}
\ \in\ \prod_{\pi,k}\Prob\bigl(Y^{1:k}\bigr),
\label{eq:predictive-state}
\end{equation}
\begin{equation}
h \sim_{\Pi,K} h' \iff P^{\pi,k}(\cdot\mid h) = P^{\pi,k}(\cdot\mid h')
\quad\text{for all }\pi\in\Pi\text{ and }k\in K.
\label{eq:predictive-equivalence}
\end{equation}

\begin{defn}{3.1}{cross-modal predictive space}
The perceptual state space induced by $(\Pi,K)$ is the image $S(\Hist)$, equivalently
the quotient $\Hist/{\sim_{\Pi,K}}$, equipped with the smallest $\sigma$-algebra
making every coordinate prediction measurable.
\end{defn}

\begin{thm}{3.2}{minimal predictive quotient}
Write $\sigma(S) = S^{-1}\bigl(\Borel(\prod_{\pi,k}\Prob(Y^{1:k}))\bigr)\subseteq\Borel(\Hist)$.
Then:
\begin{enumerate}[label=(\roman*),leftmargin=2.2em]
\item every $P^{\pi,k}(\cdot\mid\cdot)$ is $\sigma(S)$-measurable;
\item if $z\colon \Hist\to Z$ is Borel and every $P^{\pi,k}(\cdot\mid h)$ equals
$r_{\pi,k}(z(h))$ for Borel $r_{\pi,k}$, then $\sigma(S)\subseteq\sigma(z)$ and there
is a Borel $r$ with $S=r\circ z$;
\item $\sigma(S)$ is countably generated;
\item any two minimal sufficient sub-$\sigma$-algebras agree modulo the null sets of
any fixed prior on $\Hist$;
\item $S(\Hist)$ is analytic, and if $S(\Hist)$ is Borel --- which holds when
$\Pi\times K$ is finite, or when $\Hist$ is compact metrisable and $S$ is continuous
--- then $S$ descends to a Borel isomorphism of $\Hist/{\sim_{\Pi,K}}$ onto
$S(\Hist)$ by the Lusin--Souslin theorem, so the quotient is standard Borel. Without
such a hypothesis the quotient is well defined as a $\sigma$-algebra but need not be
standard Borel.
\end{enumerate}
\end{thm}

\begin{proof}
(i) Each required law is a coordinate projection of $S$, hence
$\sigma(S)$-measurable. (ii) If $z$ is sufficient, write
$P^{\pi,k}(\cdot\mid h)=r_{\pi,k}(z(h))$; countability of $\Pi\times K$ allows the
coordinate maps to be assembled into one Borel product map $r$ with $S=r\circ z$, and
$\sigma(S)=\sigma(r\circ z)\subseteq\sigma(z)$. (iii) Each $\Prob(Y^{1:k})$ is
standard Borel, hence countably generated, and a countable product of countably
generated $\sigma$-algebras is countably generated. (iv) Two minimal sufficient
$\sigma$-algebras are each contained in the other modulo null sets by (ii), and
mutual almost-sure inclusion is almost-sure equality. (v) $S$ is Borel from a
standard Borel space, so its image is analytic. If the image is Borel and $S$ is
injective on the quotient --- which it is by construction --- then Lusin--Souslin
gives that the induced map is a Borel isomorphism onto its image.
\end{proof}

\begin{rem}{3.2a}{}
The reason for stating (v) carefully is that the image of a Borel map need not be
Borel, so the frequently made claim that a predictive quotient is ``a standard Borel
space'' is not automatic. What is always available is the $\sigma$-algebra
$\sigma(S)$, and every statement in this paper that treats the quotient as a space
should be read as carrying the hypothesis of~(v). Statement~(iv) is the sense in
which the minimal sufficient representation is unique: uniqueness holds modulo null
sets, not pointwise.
\end{rem}

\begin{rem}{3.2b}{what is and is not new}
For an uncontrolled process and a single trivial policy, Theorem~3.2 reduces to the
minimality of causal states in computational mechanics
\citep{shalizi2001computational}; for a finite Markov decision process it reduces to
the minimality of the stochastic bisimulation quotient
\citep{givan2003equivalence,ferns2004metrics}. The contribution here is not the
minimality argument, which is standard, but the setting: $S$ is defined over a
groupoid of interventions acting on a sheaf of typed local states, so that the same
quotient is simultaneously constrained by prediction, by cross-modal equivariance
\eqref{eq:eeq}, by local-to-global descent \eqref{eq:descent-energy} and by
cross-scale commutation \eqref{eq:escale}. A categorical treatment of sufficiency
in the same spirit, though without the interventional layer, is given by
\citet{fritz2020synthetic}. Section~\ref{sec:identifiability} shows
what this buys.
\end{rem}

\begin{figure}[htbp]
\centering
\includegraphics[width=\textwidth]{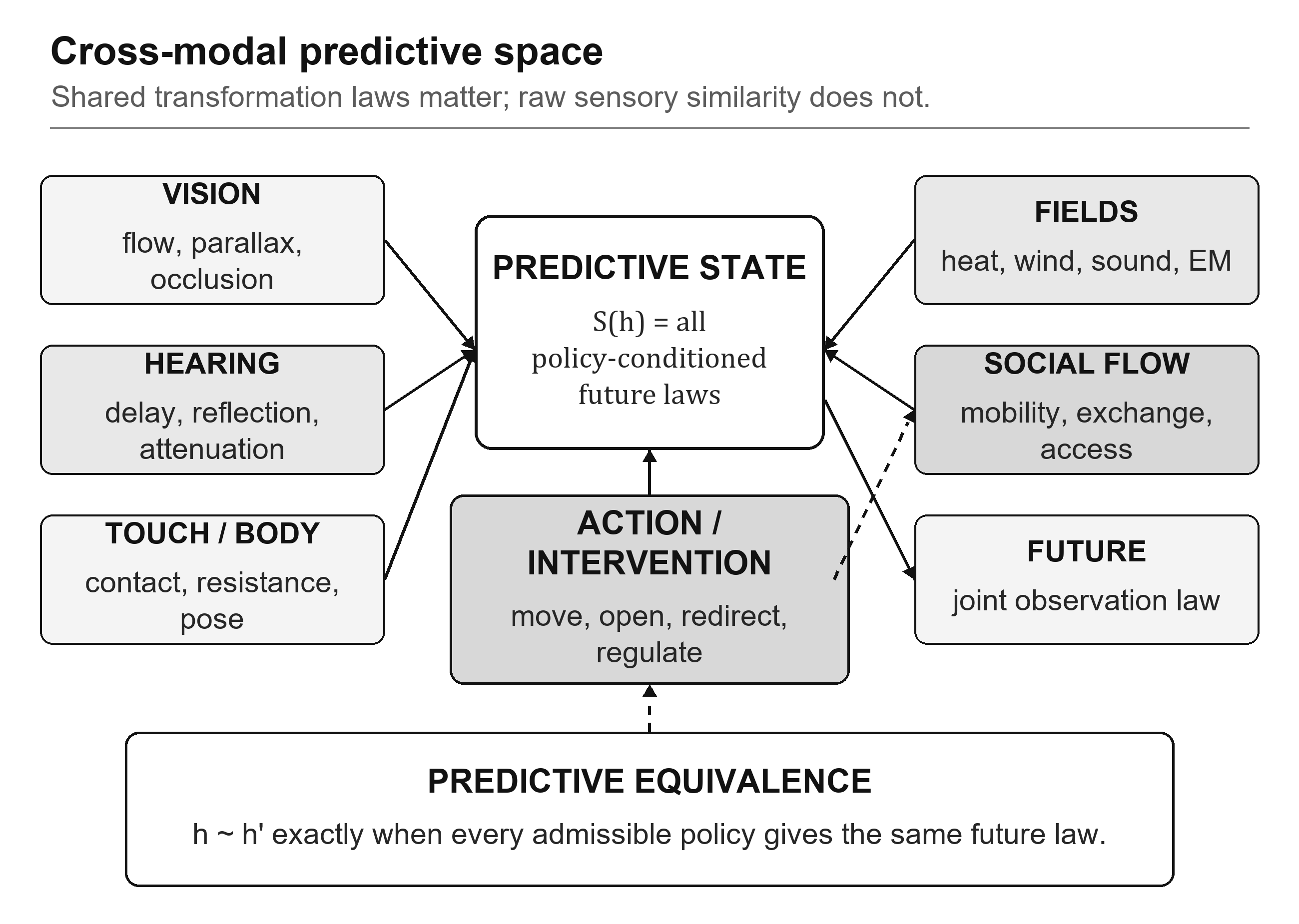}
\caption{The common space is the minimal action-conditioned predictive state
inferred from heterogeneous local observations.}
\label{fig:predictive-state}
\end{figure}

\subsection{Interventional semantics, identifiability and the status of coordinates}
\label{sec:identifiability}

Nothing in Section~3.2 is causal so far. The family $P$ was introduced as a family of
regular conditional laws, and a conditional law computed from logged data is not in
general the law that would obtain under an executed intervention. Since the
definition of space proposed here is interventional, this gap must be closed
explicitly rather than absorbed into the notation.

\subsubsection{When the predictive state is interventional}

\begin{assu}{I1}{interventional reading}
For $\pi\in\Pi$ and a history $h$, $P^{\pi,k}(\cdot\mid h)$ denotes the law of the
next $k$ multimodal observations under $\mathrm{do}(\pi)$ --- that is, under the
arrows of $G$ that $\pi$ selects, applied to the state reached by $h$ --- and not the
conditional law of observations in data generated by some other policy.
\end{assu}

\begin{assu}{I2}{sequential ignorability}
The behaviour policy $b$ that generated the data selects each arrow $g_t$ as a
measurable function of the observed history $h_t$ together with exogenous randomness
that is independent of the latent state given $h_t$. Equivalently, no unobserved
variable simultaneously drives the choice of intervention and the subsequent
observations.
\end{assu}

\begin{assu}{I3}{positivity, or interventional excitation}
There is $\varepsilon>0$ such that $b(g\mid h)\ge\varepsilon$ for every admissible
arrow $g$ at every history $h$ in the support of the data.
\end{assu}

\begin{lem}{3.5}{identification by g-computation}
Under I2 and I3, for every $\pi\in\Pi$ and $k\in K$ the interventional law
$P^{\pi,k}(\cdot\mid h)$ is identified from the observational distribution by the
g-formula
\begin{equation}
P^{\pi,k}(y_{1:k}\mid h)
=\int \prod_{j=1}^{k}
p\bigl(y_j \mid h,\, g_{1:j},\, y_{1:j-1}\bigr)\,
\pi\bigl(g_j \mid h,\, y_{1:j-1}\bigr)\, \mathrm{d}g_{1:k}.
\label{eq:gformula}
\end{equation}
Consequently the canonical predictive state $S$ of \eqref{eq:predictive-state} is
estimable from logged data. If I3 fails on a subfamily of arrows, then $S$ is
identified only for the policies supported by $b$, and Theorem~3.2 returns the
minimal sufficient statistic for that subfamily, which is a strictly coarser space.
\end{lem}

\begin{proof}
Immediate from the sequential g-computation identity for time-varying treatments
under sequential ignorability and positivity \citep[\S6]{robins1986new}; the groupoid
structure supplies the composition of arrows but plays no role in the identification
argument.
\end{proof}

\begin{rem}{3.5a}{}
Lemma~3.5 is where the philosophical claim of this paper acquires an operational
edge. Positivity is not a technical convenience. An intervention that was never
executed contributes nothing to the geometry, and a state distinction that only such
an intervention could reveal is simply not part of the space that the data define.
The frequently voiced intuition that a city ``has'' a geometry which sensing merely
uncovers is, on this account, an assertion that the relevant interventions have
positive probability under the observation regime --- an empirical claim, and a
checkable one.
\end{rem}

\subsubsection{Recovery of the latent topology}

At any fixed probe $U$, collect the observation maps into
$H=(h_1,\dots,h_m)\colon X\to\prod_\alpha Y_\alpha$. A family of modalities
\emph{jointly separates points} when, for every $x\neq x'$, some $\alpha$ satisfies
$h_\alpha(x)\neq h_\alpha(x')$. No individual modality need be injective.

\begin{lem}{3.6}{joint topological embedding}
If $X$ is compact Hausdorff, each $Y_\alpha$ is Hausdorff, every $h_\alpha$ is
continuous and the family $\{h_\alpha\}$ jointly separates points, then $H$ is a
topological embedding of $X$ into $\prod_\alpha Y_\alpha$. If moreover every
$h_\alpha$ is equivariant with respect to $\Phi_g$ and $T^g_\alpha$, then the latent
action is conjugate on $H(X)$ to the product observation action
$\prod_\alpha T^g_\alpha$.
\end{lem}

\begin{proof}
Joint point separation makes $H$ injective. A continuous injection from a compact
space into a Hausdorff space is a homeomorphism onto its image. Equivariance gives
$H\circ\Phi_g=(\prod_\alpha T^g_\alpha)\circ H$, which is conjugacy after restricting
the product action to $H(X)$.
\end{proof}

This is a statement about the true observation maps, not about a learned encoder,
and by itself it settles nothing statistically. The next subsection supplies what is
actually needed.

\subsubsection{What interventions buy: identifiability up to the centraliser}

Assume throughout:
\begin{description}[leftmargin=3.2em,style=nextline,font=\normalfont\bfseries]
\item[(E1)] $X$ is a compact connected Hausdorff space and
$\Phi\colon G\to\Homeo(X)$ is an action of a topological group $G$ by
homeomorphisms;
\item[(E2)] each $h_\alpha$ is continuous and the family jointly separates points;
\item[(E3)] equivariance holds, $h_\alpha\circ\Phi_g = T^g_\alpha\circ h_\alpha$;
\item[(E4)] \emph{interventional faithfulness}: for every $x\neq x'$ there exist
$\pi\in\Pi$ and $k\in K$ with
$P^{\pi,k}(\cdot\mid x)\neq P^{\pi,k}(\cdot\mid x')$.
\end{description}

\begin{thm}{3.7}{identifiability up to the centraliser of the intervention group}
Let $z\colon\prod_\alpha Y_\alpha\to Z$ be continuous, predictively sufficient in the
sense that every $P^{\pi,k}$ factors through $z$, and equivariant for an action
$\Psi$ of $G$ on $Z$, that is
$z\circ(\prod_\alpha T^g_\alpha)=\Psi_g\circ z$. Then under E1--E4 the composite
$\varphi := z\circ H \colon X\to Z$ is a homeomorphism onto its image and satisfies
$\varphi\circ\Phi_g = \Psi_g\circ\varphi$ for every $g$. If $\varphi'$ is any second
map with the same properties, then $\tau := \varphi'^{-1}\circ\varphi$ is a
homeomorphism of $X$ commuting with every $\Phi_g$, that is
\begin{equation}
\tau \ \in\ Z_{\Homeo(X)}\bigl(\Phi(G)\bigr),
\label{eq:centraliser}
\end{equation}
the centraliser of the action. The latent space is therefore determined up to this
centraliser, and not merely up to an arbitrary homeomorphism.
\end{thm}

\begin{proof}
By E1--E2 and Lemma~3.6, $H$ is an embedding. Suppose $z(H(x))=z(H(x'))$. Since every
$P^{\pi,k}$ factors through $z$, we get $P^{\pi,k}(\cdot\mid x)=P^{\pi,k}(\cdot\mid x')$
for all $\pi$ and $k$, so $x=x'$ by E4; hence $\varphi=z\circ H$ is injective. It is
continuous as a composite of continuous maps, and a continuous injection from a
compact space into a Hausdorff space is a homeomorphism onto its image. Equivariance
follows by composition:
\begin{equation}
\varphi\circ\Phi_g
= z\circ H\circ\Phi_g
= z\circ\Bigl(\prod_\alpha T^g_\alpha\Bigr)\circ H
= \Psi_g\circ z\circ H
= \Psi_g\circ\varphi ,
\label{eq:equivariance-chase}
\end{equation}
using E3 and the equivariance of $z$. Finally, if $\varphi$ and $\varphi'$ are both
$G$-equivariant homeomorphisms onto the same image, then
$\tau=\varphi'^{-1}\circ\varphi$ satisfies
$\tau\circ\Phi_g=\varphi'^{-1}\circ\Psi_g\circ\varphi
=\Phi_g\circ\varphi'^{-1}\circ\varphi=\Phi_g\circ\tau$,
so $\tau$ lies in the centraliser.
\end{proof}

\begin{cor}{3.8}{space as a torsor; the exact status of coordinates}
If the action $\Phi$ is simply transitive, so that $X$ is a principal homogeneous
$G$-space, then $Z_{\Homeo(X)}(\Phi(G))\cong G$ acting by right translations.
Consequently $X$ is recovered as a $G$-torsor: the observation family together with
the intervention family determines everything about the space except the choice of an
origin and a frame, and that choice is exactly the residual ambiguity --- no more and
no less.
\end{cor}

\begin{proof}
Fix $x_0$ and identify $X$ with $G$ by $g\mapsto\Phi_g x_0$; this is a homeomorphism
because the action is simply transitive. If $\tau$ commutes with every $\Phi_g$ and
$\tau(x_0)=\Phi_a x_0$, then $\tau(\Phi_g x_0)=\Phi_g\tau(x_0)=\Phi_{ga}x_0$, so
$\tau$ is right translation by $a$. Conversely every right translation commutes with
the left action.
\end{proof}

Corollary~3.8 is the precise form of the claim, made informally in Section~2.2, that
coordinate changes are representational while causal and metrical invariants are
empirical. A coordinate system is a section of the torsor. Two observers who disagree
about coordinates disagree by an element of $G$ and about nothing else, and this is a
theorem about the observational system rather than a stipulation.

\begin{prop}{3.9}{quantitative version}
Let $(X,d)$ be a compact metric space and define the \emph{interventional separation
modulus}
\begin{equation}
\eta(\tau) \;:=\; \inf\Bigl\{\,
\sup_{\pi,k}\ \TV\bigl(P^{\pi,k}(\cdot\mid x),\,P^{\pi,k}(\cdot\mid x')\bigr)
\ :\ d(x,x')\ge\tau \,\Bigr\}.
\label{eq:separation-modulus}
\end{equation}
Suppose an encoder $z$ and decoder family $Q$ satisfy
$\sup_{\pi,k}\TV\bigl(Q^{\pi,k}(\cdot\mid z(H(x))),\,P^{\pi,k}(\cdot\mid x)\bigr)\le\epsilon$
for every $x$. Then every fibre of $z\circ H$ has $d$-diameter at most
$\tau_\epsilon := \inf\{\tau>0:\eta(\tau)>2\epsilon\}$. Exact recovery is the limiting
case $\epsilon\to0$ with $\eta(\tau)>0$ for every $\tau>0$, which is E4.
\end{prop}

\begin{proof}
If $z(H(x))=z(H(x'))$ then for all $\pi$ and $k$ the triangle inequality for total
variation gives $\TV(P^{\pi,k}(\cdot\mid x),P^{\pi,k}(\cdot\mid x'))\le2\epsilon$.
Were $d(x,x')\ge\tau$ with $\eta(\tau)>2\epsilon$, the definition of $\eta$ would
force the left-hand side to exceed $2\epsilon$, a contradiction.
\end{proof}

Proposition~3.9 converts the predictive dimension of Section~3.5 into a resolution
statement: an encoder that predicts within $\epsilon$ resolves space to within
$\tau_\epsilon$, and the function $\epsilon\mapsto\tau_\epsilon$ is the operating
characteristic of the observational system. It is estimable, because $\eta$ can be
lower-bounded empirically by executing pairs of policies from nearby states.

\begin{rem}{3.10}{relation to interventional causal representation learning, and a sanity check}
Theorem~3.7 is the topological counterpart of the identifiability results proved for
structural causal models under interventions
\citep{brehmer2022weakly,ahuja2023interventional,vonkugelgen2023nonparametric,lippe2022citris,squires2023linear,vonkugelgen2021self},
which recover latent causal variables up to permutation and elementwise
reparameterisation. Here the residual ambiguity is expressed intrinsically as the
centraliser of the intervention group; it specialises to permutation-and-scaling when
the action is generated by independent single-coordinate interventions. The present
statement is weaker in one respect that should be stated plainly: equivariance of the
learned encoder is assumed in E-form rather than derived, and in practice it is
enforced by the penalty \eqref{eq:eeq} and verified as a residual, not guaranteed.
The corresponding sanity check is instructive. If the intervention family is removed
--- $G$ trivial and $\Pi$ a singleton --- then E4 fails for every pair of states not
already separated by a single observation, the centraliser becomes all of
$\Homeo(X)$, and Theorem~3.7 asserts nothing. This is exactly the impossibility of
unsupervised disentanglement without inductive bias \citep{locatello2019challenging}
and the non-identifiability of nonlinear ICA without auxiliary structure
\citep{hyvarinen2019nonlinear,khemakhem2020variational}. Partial recoveries are
possible from multiple sufficiently distinct views or from spatial dependence
\citep{gresele2020incomplete,halva2024identifiable}, but each such result buys
identifiability by importing exactly the kind of auxiliary structure that an
intervention family supplies here. The theorem therefore does
not merely tolerate interventions; it is empty without them, which is the strongest
form in which the paper's central thesis can be stated.
\end{rem}

Two caveats remain. Compactness in E1 excludes unbounded latent spaces and is used
only to convert continuous injectivity into homeomorphism; it can be replaced by
properness of $\varphi$. Joint separation in E2 is a strong assumption in any real
deployment, and Section~\ref{sec:limitations} treats its failure as the principal
identifiability risk rather than as an edge case.

\subsection{Distance as action cost}

For an arrow $u\colon x\to x'$, let $A(u)\ge0$ be its cost. Composition is executable
succession. Define
\begin{equation}
c_A(x,x') \;=\; \inf\{\,A(u)\ :\ u\colon x\to x'\ \text{in }G\,\},
\qquad \inf\emptyset := +\infty.
\label{eq:action-cost}
\end{equation}

\begin{prop}{3.4}{directed action geometry}
If $A(\mathrm{id}_x)=0$ and $A(v\circ u)\le A(u)+A(v)$, then $c_A$ is an extended
directed quasi-metric: $c_A(x,x)=0$ and $c_A(x,z)\le c_A(x,y)+c_A(y,z)$. It becomes a
metric only if finite mutual reachability, definiteness and reversal symmetry are
added.
\end{prop}

This correction matters in cities. Walking uphill and downhill, travelling with and
against congestion, gaining and losing institutional access, and buying and selling
in illiquid markets are generally asymmetric. Calling every such cost a metric
conceals the very geometry one wishes to study. Finsler and directed network
geometries are often the appropriate intermediate objects \citep{bao2000introduction}.

\subsection{Predictive dimension}

Dimension is not taken as the definition of space. It quantifies the complexity of a
chosen predictive representation. Without regularity restrictions a real number can
measurably encode pathological amounts of information, so the encoder class must be
fixed. Let $\mathcal{Z}_{d,L}$ be an admissible class of $d$-dimensional
$L$-Lipschitz encoders, let $D$ be a divergence, and let $Q$ denote a decoder family.
\begin{equation}
\Dpred(\epsilon;\Pi,K,L)
=\inf\Bigl\{ d \ :\ \exists\, z\in\mathcal{Z}_{d,L},\, Q,\ \
\sup_{\pi,k}\ \Ex\, D\bigl(P^{\pi,k}(\cdot\mid \Hist)\,\big\|\,Q^{\pi,k}(\cdot\mid z(\Hist))\bigr)\le\epsilon \Bigr\}.
\label{eq:dpred}
\end{equation}
Equation~\eqref{eq:dpred} makes dimension task-, scale-, error- and model-relative.
Topological, Hausdorff, spectral, information and predictive dimensions may disagree
because they answer different questions. Their comparison is informative; their
identification is not.
\section{A sheaf-valued geometry of the city}

\subsection{Stratified base and fibres}

Let $B$ be a finite regular cell complex or a Whitney-stratified space representing
rooms, air volumes, walls, fa\c{c}ades, streets, rails, pipes, junctions and sensors
\citep{pflaum2001analytic,edelsbrunner2010computational}. Let $\pi\colon E\to B\times\Reals$
be a bundle-like projection. The fibre over a cell and time is not one homogeneous
coordinate vector but a typed product of local geometric, physical, mobility, social
and economic state spaces. Different strata can have different dimensions. Singular
junctions are not defects to be smoothed away; they are where transfer, impedance and
governance conditions are imposed.

A cellular sheaf $\Shf$ assigns a stalk $\Shf(\sigma)$ to every cell $\sigma$ and a
restriction map $\rho_{\sigma\le\tau}\colon\Shf(\sigma)\to\Shf(\tau)$ to every
incidence relation, with functorial compatibility. A section is a choice of local
states. It is global when all incident choices agree after restriction. This is
precisely the right logic for combining sensors, models and institutions whose local
descriptions overlap but are not numerically identical
\citep{robinson2017sheaves,hansen2019spectral,bodnar2022neural}; a systematic
account of cellular sheaves and cosheaves is given by \citet{curry2014sheaves}.

\begin{figure}[htbp]
\centering
\includegraphics[width=\textwidth]{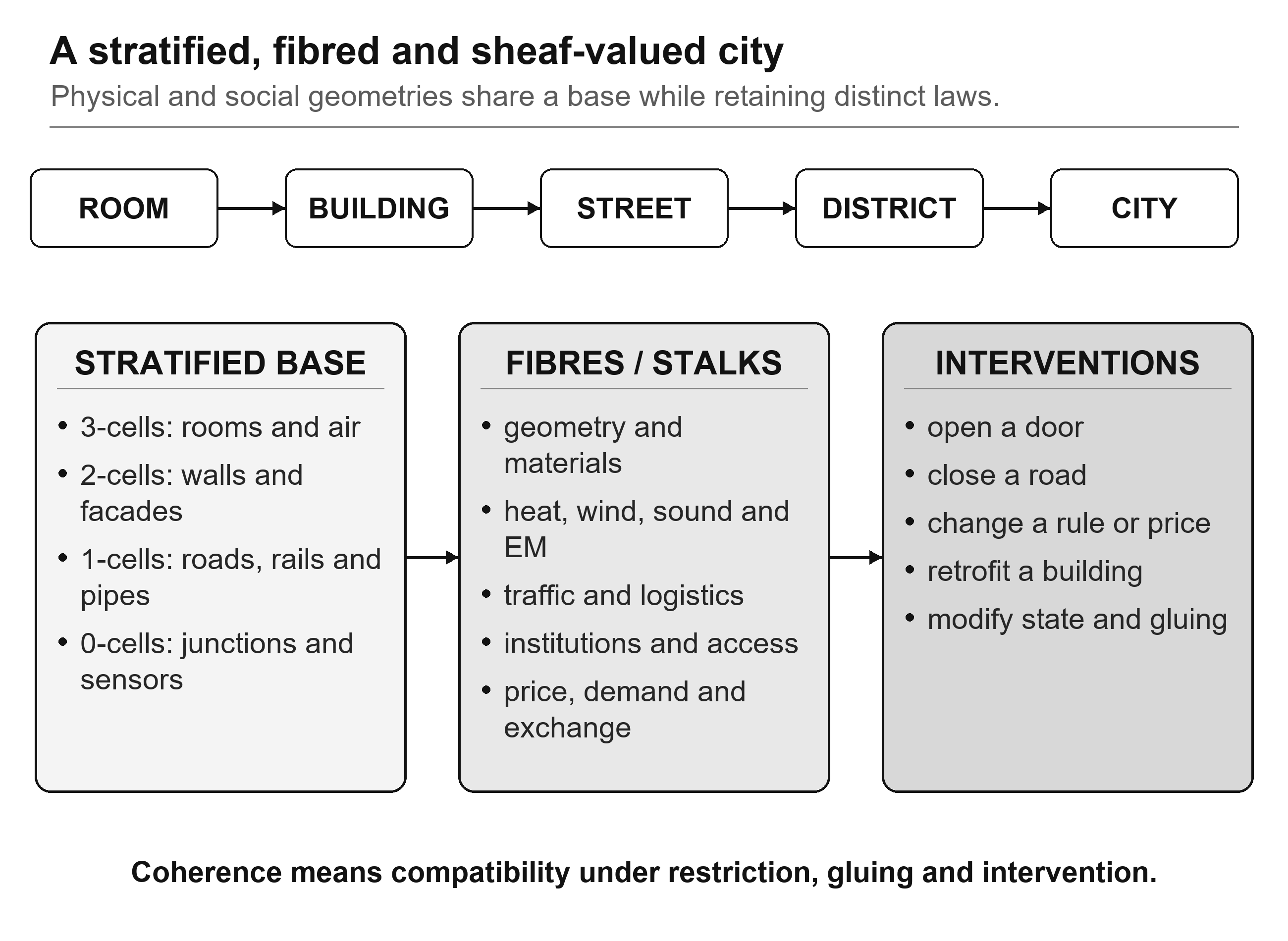}
\caption{A city is represented by a stratified base, typed local stalks, restriction
maps, cross-scale coarsening and interventions that may change both state and
gluing.}
\label{fig:city}
\end{figure}

\subsection{Descent energy, the sheaf Laplacian and holonomy}

For a finite cellular sheaf with inner products on stalks, let $C^0(B;\Shf)$ be the
space of $0$-cochains and $\delta_\Shf$ the sheaf coboundary. Let $W$ be positive
definite on incidence residuals.
\begin{equation}
L_\Shf = \delta_\Shf^{*}\, W\, \delta_\Shf,
\qquad
\Edesc(z) = \bigl\|W^{1/2}\delta_\Shf z\bigr\|^2 = \langle z, L_\Shf z\rangle .
\label{eq:descent-energy}
\end{equation}

\begin{prop}{4.1}{consistency certificate}
$\Edesc(z)=0$ if and only if $z$ is a global section. Equivalently,
$\ker L_\Shf=\ker\delta_\Shf$.
\end{prop}

\begin{proof}
Positive definiteness of $W$ gives $\Edesc(z)=0$ exactly when $\delta_\Shf z=0$. The
kernel of the coboundary is the space of global compatible sections, and
$\langle z,L_\Shf z\rangle=\langle \delta_\Shf z, W\delta_\Shf z\rangle$.
\end{proof}

The scalar graph Laplacian is recovered when every stalk is one-dimensional and every
restriction is the identity. Nontrivial restriction maps allow neighbouring cells to
translate between coordinate frames, sensor types, institutional categories or learned
feature bases. Neural sheaf diffusion exploits precisely this additional geometry
\citep{bodnar2022neural}.

Proposition~4.1 as stated is an identity rather than a discovery, and on its own it
would not justify the sheaf machinery: the same certificate for a graph Laplacian with
identity restrictions is equally immediate. What distinguishes a sheaf from a
multilayer network is that non-identity restriction maps carry \emph{holonomy}, and
holonomy is an obstruction that no identity-restriction model can represent.

\begin{prop}{4.2}{holonomy obstruction}
Let $\Shf$ be a cellular sheaf on a connected graph $B$ with stalks $\Reals^n$ and
restriction maps in $O(n)$, so that
$(\delta_\Shf z)_e = \rho_{j,e}z_j-\rho_{i,e}z_i$ for each edge $e=(i,j)$. Fix a base
vertex $v_0$ and let $\Hol(\Shf,v_0)\subseteq O(n)$ be the group generated by the
holonomies of all cycles based at $v_0$. Then
\begin{equation}
\dim\ker L_\Shf \;=\; \dim\Fix\bigl(\Hol(\Shf,v_0)\bigr),
\label{eq:holonomy}
\end{equation}
the dimension of the subspace of $\Reals^n$ fixed pointwise by the holonomy group. In
particular, with identity restrictions the holonomy is trivial and
$\dim\ker L_\Shf = n$ per connected component, which is the graph Laplacian case;
whereas a single cycle with nontrivial holonomy can force $\dim\ker L_\Shf = 0$, so
that the sheaf admits no nonzero global section even though the graph is connected and
every local datum is individually consistent.
\end{prop}

\begin{proof}
A $0$-cochain lies in $\ker\delta_\Shf$ precisely when it is parallel, that is
$z_j = \rho_{j,e}^{-1}\rho_{i,e}z_i$ along every edge. Parallel transport along a path
is then determined by the restriction maps, so a parallel section is determined by its
value at $v_0$, and such a value extends consistently if and only if it is fixed by the
holonomy of every cycle at $v_0$. The assignment $z\mapsto z(v_0)$ is therefore a
linear isomorphism from $\ker\delta_\Shf$ onto $\Fix(\Hol(\Shf,v_0))$, and
$\ker L_\Shf=\ker\delta_\Shf$ by Proposition~4.1.
\end{proof}

This is the sheaf-theoretic form of the connection Laplacian of vector diffusion maps
and of the graph connection Laplacian
\citep{singer2012vector,bandeira2013cheeger}, and it is what gives Proposition~4.1
empirical content. A numerical instance is reported in
Section~\ref{sec:experiments}: on an eight-cycle with $\Reals^2$ stalks and rotation
restrictions, the graph Laplacian has a two-dimensional kernel regardless of the data,
the sheaf with trivial holonomy also has a two-dimensional kernel, and the sheaf with
holonomy angle $0.70$ has kernel dimension zero with spectral gap
$7.7\times10^{-3}$. The same experiment exhibits the converse failure mode: a field
that is genuinely compatible with the correct sheaf receives descent energy
$1.0\times10^{-16}$, while the identity-restriction graph model assigns it energy
$1.31$ and so reports a spurious inconsistency.

The urban reading is direct. Suppose a vector-valued layer --- wind velocity, acoustic
intensity gradient, pedestrian flow --- is recorded in local frames aligned to street
orientation, as is normal practice. The restriction maps across an incidence are the
rotations between adjacent street orientations, and a closed circuit of streets whose
orientations do not compose to the identity has nontrivial holonomy. A model that
averages raw local vectors, which is what a graph Laplacian does, will then report
systematic residuals at exactly those circuits, and will do so whether or not the
physical field is compatible. The sheaf model makes a falsifiable prediction here: the
residual of the identity-restriction model at a circuit should scale with the holonomy
angle of that circuit, and should vanish for the sheaf model.
Section~\ref{sec:experimental-programme} lists this as Experiment~F.

\begin{prop}{4.3}{restriction maps are identifiable, and the constraint removes an inconsistency}
Suppose paired local states $\{(z_i^{(m)},z_j^{(m)})\}_{m=1}^{M}$ are observed across
an incidence with $z_j = \rho z_i + \text{noise}$ and $\rho\in O(n)$. Then the
orthogonal Procrustes estimator $\hat\rho = UV^{\mathsf{T}}$, where
$\sum_m z_j^{(m)}(z_i^{(m)})^{\mathsf{T}} = U\Sigma V^{\mathsf{T}}$, is consistent.
Moreover, when $z_i$ is itself measured with independent error of variance $\sigma^2$
--- the errors-in-variables regime that is the norm in sensor networks --- the
unconstrained least-squares estimator is \emph{inconsistent}, converging to the
attenuated map $\lambda(\lambda+\sigma^2)^{-1}\rho$ with $\lambda$ the per-component
signal variance, whereas the projection onto $O(n)$ discards precisely this
multiplicative attenuation and remains consistent.
\end{prop}

\begin{proof}
Consistency of the Procrustes estimator follows from the strong law applied to the
cross-moment matrix together with continuity of the polar factor at a nonsingular
limit. For the second claim, the least-squares limit is
$(\Sigma_{xx}+\sigma^2 I)^{-1}\Sigma_{xy}=\lambda(\lambda+\sigma^2)^{-1}\rho$ for
isotropic signal, which is a positive scalar multiple of $\rho$ and therefore has the
same polar factor; the $O(n)$ projection of the attenuated map is $\rho$ exactly.
\end{proof}

The practical consequence is stronger than regularisation. The equivariance and
descent constraints are usually presented as priors that trade bias for variance, and
would then be expected to help only in the small-sample regime. Proposition~4.3 says
something else: under measurement error in the inputs, the unconstrained estimator has
an asymptotic bias floor that no amount of data removes, and the structural constraint
removes it. The numerical check in Section~\ref{sec:experiments} confirms this, with
the unconstrained error plateauing at the predicted floor $1.9\times10^{-4}$ at
$\sigma=0.1$ and $2.7\times10^{-3}$ at $\sigma=0.2$ while the constrained error
continues to decrease with sample size.

\subsection{Coupled field and flow dynamics}

Let $z_t$ be a sheaf-valued urban state and $u_t$ an intervention. A deliberately
general evolution law is
\begin{equation}
M(z_t)\,\dot z_t \;=\; -L_\Shf(z_t)\,z_t + N(z_t) + B(z_t)u_t + s_t + \xi_t .
\label{eq:dynamics}
\end{equation}
The block operator $L_\Shf$ carries diffusion and compatibility terms; $N$ contains
nonlinear transport, reaction, behavioural response and market dynamics; $B$ specifies
actuators; $s$ and $\xi$ are forcing and uncertainty. Physical conservation laws may
occupy selected blocks. Social and economic layers need not be forced into mass
conservation. Their locality is instead encoded by stalks, restrictions, reachability
and intervention response.
\begin{equation}
\mathrm{do}(u)\ \colon\ (\Shf,\ \rho,\ \Phi)\ \longmapsto\ (\Shf_u,\ \rho_u,\ \Phi_u).
\label{eq:do}
\end{equation}
Equation~\eqref{eq:do} allows an intervention to change not only the current state but
the coupling architecture. Opening a door alters acoustic, thermal and aerodynamic
transfer maps. Closing a road alters mobility restrictions. A zoning or pricing rule
may alter economic accessibility without moving any wall. This is where causal
intervention and sheaf gluing meet.

\subsection{Cross-scale compatibility}

Let $R_{\ell\to m}$ coarsen a state from level $\ell$ to $m$. Exact commutation is
exceptional; a useful model measures the defect.
\begin{equation}
\Escale(\ell,m;u) \;=\; \Ex_z\bigl\| R_{\ell\to m}\,\Phi^\ell_u(z)
\;-\; \Phi^m_u\bigl(R_{\ell\to m} z\bigr) \bigr\|^2 .
\label{eq:escale}
\end{equation}
Small residual means that acting and then aggregating nearly agrees with aggregating
and then acting. Large residual identifies a scale at which omitted heterogeneity
matters. The city is therefore not assumed to possess one privileged resolution. It is
an inverse or multiresolution system whose transition maps are themselves testable.

\section{Audit of the core formulae}

The following audit separates statements that are identities, statements that require
hypotheses and statements that should be estimated as residuals. This prevents
suggestive notation from being mistaken for a theorem.

\begingroup
\small
\setlength{\tabcolsep}{4pt}
\begin{longtable}{L{2.15cm} L{1.95cm} L{4.15cm} L{4.25cm}}
\caption{Type and falsifiability audit of the principal formulae.}
\label{tab:audit}\\
\toprule
\textbf{Object} & \textbf{Mathematical status} & \textbf{Conditions} & \textbf{Failure mode / test}\\
\midrule
\endfirsthead
\toprule
\textbf{Object} & \textbf{Mathematical status} & \textbf{Conditions} & \textbf{Failure mode / test}\\
\midrule
\endhead
\bottomrule
\endfoot
Predictive state \eqref{eq:predictive-state} & Well-typed product of probability laws
& Standard Borel histories and futures; regular conditional laws exist
& Uncountable policy classes need extra measurable structure; check countability of $\Pi\times K$.\\
\addlinespace
Predictive quotient \eqref{eq:predictive-equivalence} & Minimal sufficient statistic
& Every required future law is indexed in $S$; null sets treated modulo a prior
& A restricted policy family may merge states that another intervention separates.\\
\addlinespace
Joint embedding (Lemma 3.6) & Conditional theorem
& Compact Hausdorff latent space; continuous maps into Hausdorff targets; joint separation
& Unknown nonlinear encoders remain statistically non-identifiable without further structure.\\
\addlinespace
Equivariance \eqref{eq:eeq} & Loss / defect, not assumed equality
& Both compositions have the same source and target; modality metrics declared
& Coordinate mismatch or misaligned action timestamps increase the residual.\\
\addlinespace
Interventional identification (Lemma 3.5) & Conditional theorem
& Sequential ignorability (I2) and positivity (I3); admissible arrows have behaviour probability bounded below
& Unobserved confounding or zero-probability interventions make the quotient statistical, not causal; test by policy-support diagnostics and negative controls.\\
\addlinespace
Centraliser identifiability (Theorem 3.7) & Conditional theorem
& Compact connected latent space; joint separation; equivariant learned encoder; interventional faithfulness (E4)
& Removing interventions makes the centraliser all of $\Homeo(X)$ and the statement vacuous, recovering \citep{locatello2019challenging}; test E4 by the separation modulus of Proposition~3.9.\\
\addlinespace
Action cost \eqref{eq:action-cost} & Directed extended quasi-metric
& Identity has zero cost; composition is subadditive
& Without reversal symmetry it is not a metric; test asymmetry explicitly.\\
\addlinespace
Predictive dimension \eqref{eq:dpred} & Model-relative complexity index
& Encoder regularity, policy set, horizon, divergence and tolerance all declared
& Without regularity, pathological encodings make dimensional claims vacuous.\\
\addlinespace
Sheaf energy \eqref{eq:descent-energy} & Exact consistency certificate
& Finite cellular sheaf; positive-definite $W$; correct restriction maps
& A wrong sheaf can report false consistency; validate restrictions first.\\
\addlinespace
Holonomy obstruction \eqref{eq:holonomy} & Exact theorem
& Cellular sheaf on a connected graph with $O(n)$ restriction maps
& Misspecified restriction maps produce spurious or missing obstructions; validate by Procrustes recovery (Proposition~4.3) before interpreting the kernel.\\
\addlinespace
Cross-scale law \eqref{eq:escale} & Empirical residual
& Coarsening map and norm specified; comparable interventions at both levels
& High-frequency or singular effects can make aggregation and action fail to commute.\\
\addlinespace
Projective limit (Theorem 8.4) & Exact theorem; saturation clause is empirical
& Directed context poset; countable cofinal chain for the Borel clause
& Saturation (iv) may fail: a curve that never flattens means no finite instrument set determines the space. Report $\Esat$ and separate noise-driven from refinement-driven decrease.\\
\end{longtable}
\endgroup

\subsection{Reproducible synthetic study under noise}
\label{sec:experiments}

To check the algebra without disguising a toy model as urban evidence, consider a
latent state $(x,y)$ on the two-torus. Vision observes $v(x)=(\cos x,\sin x)$,
audition observes $a(y)=(\cos y,\sin y)$, and touch observes
$t(x,y)=(\cos(x+y),\sin(x+y))$; a fourth modality $w(x,y)=(\cos(x-y),\sin(x-y))$ is
held in reserve for the refinement test. Each view is partial; the
visual--auditory pair separates points. Six translation actions are sampled. For every
action the observation transformation is a planar rotation, so exact equivariance is
known:
\begin{equation}
v(x+\Delta x)=R_{\Delta x}\,v(x),\qquad
a(y+\Delta y)=R_{\Delta y}\,a(y),\qquad
t(x+\Delta x,\,y+\Delta y)=R_{\Delta x+\Delta y}\,t(x,y).
\label{eq:torus-actions}
\end{equation}

The earlier version of this test was run noiselessly, which made every reported figure
either machine precision or an artefact of the fitting protocol, and left an inversion
in the ablation ordering that we now resolve. All experiments below therefore add
isotropic Gaussian observation noise of standard deviation $\sigma$ to every modality
at both the input and the target, and every entry is reported as the mean and standard
deviation over twenty seeds derived from the base seed \texttt{20260731}. Training uses
$12{,}000$ states and testing $5{,}000$ held-out states per seed. Linear least squares
is fitted separately for each action, and joint features include the bilinear
visual--auditory products needed to reconstruct touch.

Six checks are reported. (a)~Cross-modal prediction under the five representations of
Table~\ref{tab:results}. (b)~The value of the equivariance constraint, comparing an
unconstrained least-squares operator with its projection onto $O(2)$, across sample
sizes and noise levels. (c)~The holonomy obstruction of Proposition~4.2 on an
eight-cycle sheaf. (d)~Recovery of restriction maps by Procrustes under noise, per
Proposition~4.3. (e)~The cross-scale residual \eqref{eq:escale} as a function of the
wavenumber of the underlying field. (f)~The saturation of the projective system of
Theorem~8.4, obtained by refining the observational context from $\{v\}$ to $\{v,a\}$
to $\{v,a,t\}$ to $\{v,a,t,w\}$.

\begin{table}[htbp]
\centering
\small
\caption{Held-out prediction error in the synthetic cross-modal experiment, with
observation noise $\sigma=0.05$, mean $\pm$ standard deviation over 20 seeds.}
\label{tab:results}
\begin{tabularx}{\textwidth}{L{3.3cm} l X}
\toprule
\textbf{Representation} & \textbf{Held-out MSE} & \textbf{Interpretation}\\
\midrule
Visual only $+$ action & $0.3373\pm0.0003$ & Cannot recover the unobserved auditory phase.\\
Auditory only $+$ action & $0.3371\pm0.0004$ & Cannot recover the unobserved visual phase.\\
Joint, action-blind & $0.0817\pm0.0005$ & Averages incompatible transition operators.\\
Joint, misaligned actions & $0.0820\pm0.0005$ & Indistinguishable from action-blind (Welch $t=-1.94$); both marginalise the action set.\\
Joint $+$ correct action & $0.00584\pm0.00005$ & Approaches the irreducible noise floor $\sigma^2=0.0025$.\\
\addlinespace
Joint $+$ correct action \newline (noiseless control) & $1.2\times10^{-30}\pm9\times10^{-31}$ & Recovers the deterministic law to numerical precision.\\
\bottomrule
\end{tabularx}
\end{table}

\begin{figure}[htbp]
\centering
\includegraphics[width=\textwidth]{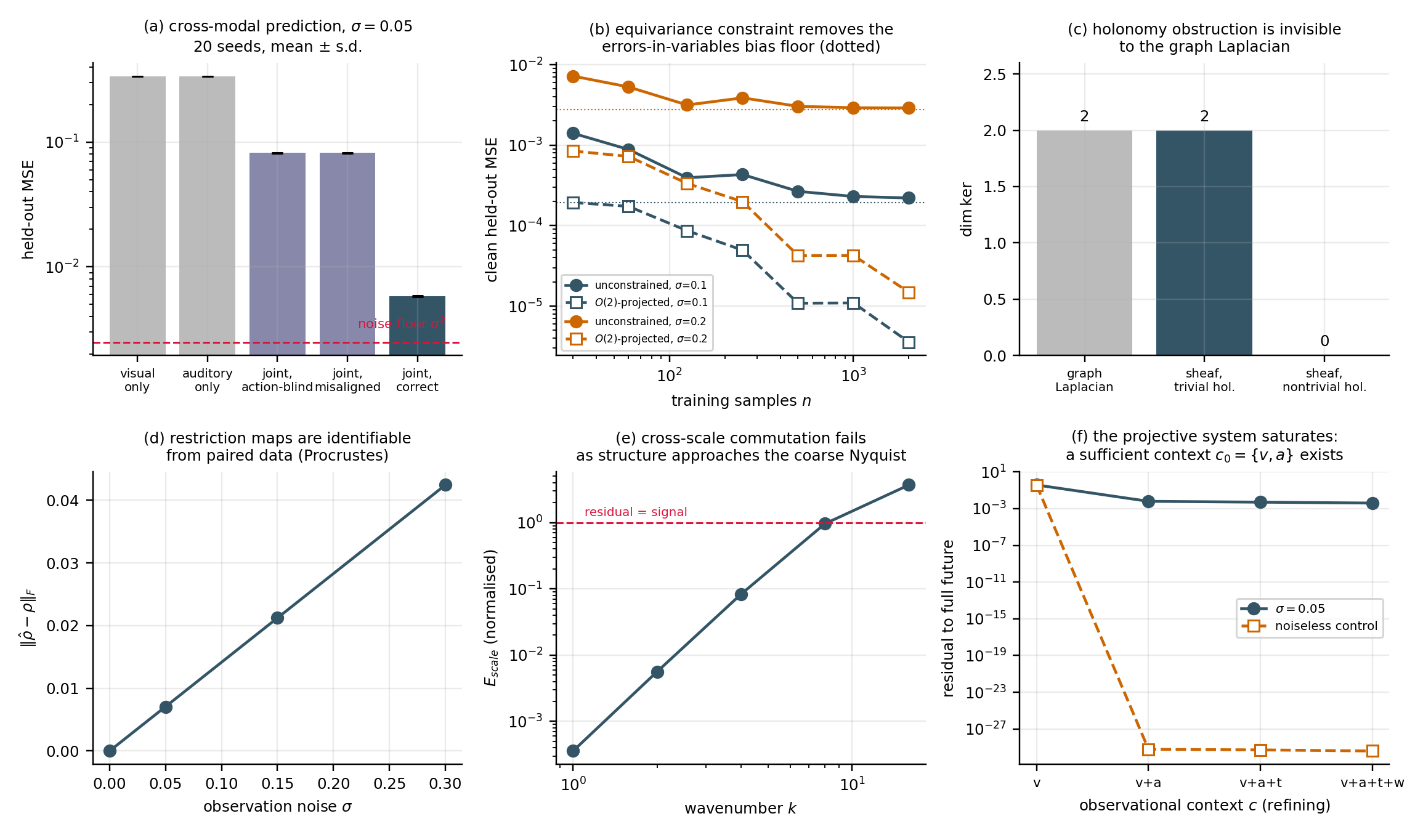}
\caption{Numerical checks under noise, twenty seeds. (a)~Exact recovery requires both
joint observation and correctly aligned actions; the action-blind and misaligned
conditions are statistically indistinguishable. (b)~The $O(2)$ constraint removes the
errors-in-variables bias floor of the unconstrained estimator (dotted lines, predicted
analytically), so the benefit persists asymptotically rather than vanishing with
sample size. (c)~A holonomy obstruction reduces the sheaf kernel to zero while the
graph Laplacian kernel is insensitive to it. (d)~Restriction maps are recoverable from
paired data with error linear in the noise. (e)~Cross-scale commutation degrades as
the field approaches the coarse Nyquist wavenumber. (f)~The projective system of
predictive quotients saturates at $c_0=\{v,a\}$; the further decrease under noise
reflects variance reduction, not refinement, as the noiseless control shows.}
\label{fig:numerics}
\end{figure}

Two results deserve comment because they correct the earlier presentation. First, in
Table~\ref{tab:results} the action-blind and misaligned-action conditions are
statistically indistinguishable, with a Welch statistic of $-1.94$ over twenty seeds.
This is the expected outcome and not a defect: both conditions marginalise over the
action set, one by omitting it and one by destroying its correspondence with the
observations, so both estimate the same average operator. The earlier report of a small
advantage for the misaligned condition was seed noise on a single run. The informative
contrast is between either of these and the correctly aligned condition, which is
fourteen times better and approaches the irreducible noise floor
$\sigma^2=2.5\times10^{-3}$.

Second, the refinement test in panel~(f) supports the saturation clause of
Theorem~8.4 but only in the noiseless control, and the distinction matters. Under noise
the residual falls by $3.7\times10^{-1}$ from $\{v\}$ to $\{v,a\}$ and then by roughly
$1\times10^{-3}$ at each further refinement; those later reductions are consistent with
variance reduction from redundant noisy views rather than with genuine refinement of
the quotient. In the noiseless control the residual falls to below $10^{-29}$ at
$\{v,a\}$ and does not fall further, which is the stationarity condition of
Theorem~8.4(iv) with sufficient context $c_0=\{v,a\}$. The methodological lesson is
that a saturation curve measured under noise cannot by itself certify sufficiency, and
an empirical protocol must separate the two effects by holding predictive capacity
fixed while adding modalities.

The result establishes internal consistency only. It does not show that a city admits
the chosen sheaf, that its policies are known, or that a robot will learn the correct
latent state. Those claims require field and embodied experiments designed to break the
model.

\section{From embodied to em-spaced intelligence}

\subsection{Two meanings of ESI}

Two recent uses of the initials ESI must be distinguished. \emph{Embodied Spatial
Intelligence} places the agent inside a perception-action loop and asks it to acquire
evidence through movement. ESI-Bench makes this requirement explicit and reports that
active exploration outperforms passive observation, while poor action selection causes
cascading perceptual failure \citep{hong2026esibench}. \emph{Em-Spaced Intelligence}
places intelligence partly in the spatial environment itself; the term is formed by
analogy with \emph{embodied}, intelligence being em-bedded in and partly constituted by
the instrumented space, and it is unrelated to the typographic em space. In that
programme, Artificially Evolved Spatial Organisms combine multimodal foundation models,
graph neural networks and non-Euclidean manifold geometry, and a Dynamic Space Protocol
mediates the co-evolution of robots and urban space \citep{yang2026genesis}. The
overlap with embodied spatial intelligence is substantive, but the ontological emphasis
differs, and the framework of this paper is neutral between them: both are instances of
Definition~6.1 with different allocations of sensing and actuation between the two
bodies.

\begin{defn}{6.1}{em-spaced intelligent system}
An em-spaced intelligent system is a coupled pair $(A,E)$ of a mobile agent $A$ and a
persistent spatial body $E$, together with a shared predictive state $S$, bidirectional
observation maps, agent actions $\pi$, environmental interventions $u$, and governance
constraints $\Gamma$. Both $A$ and $E$ can alter the future field of admissible
interaction.
\end{defn}

\begin{figure}[htbp]
\centering
\includegraphics[width=\textwidth]{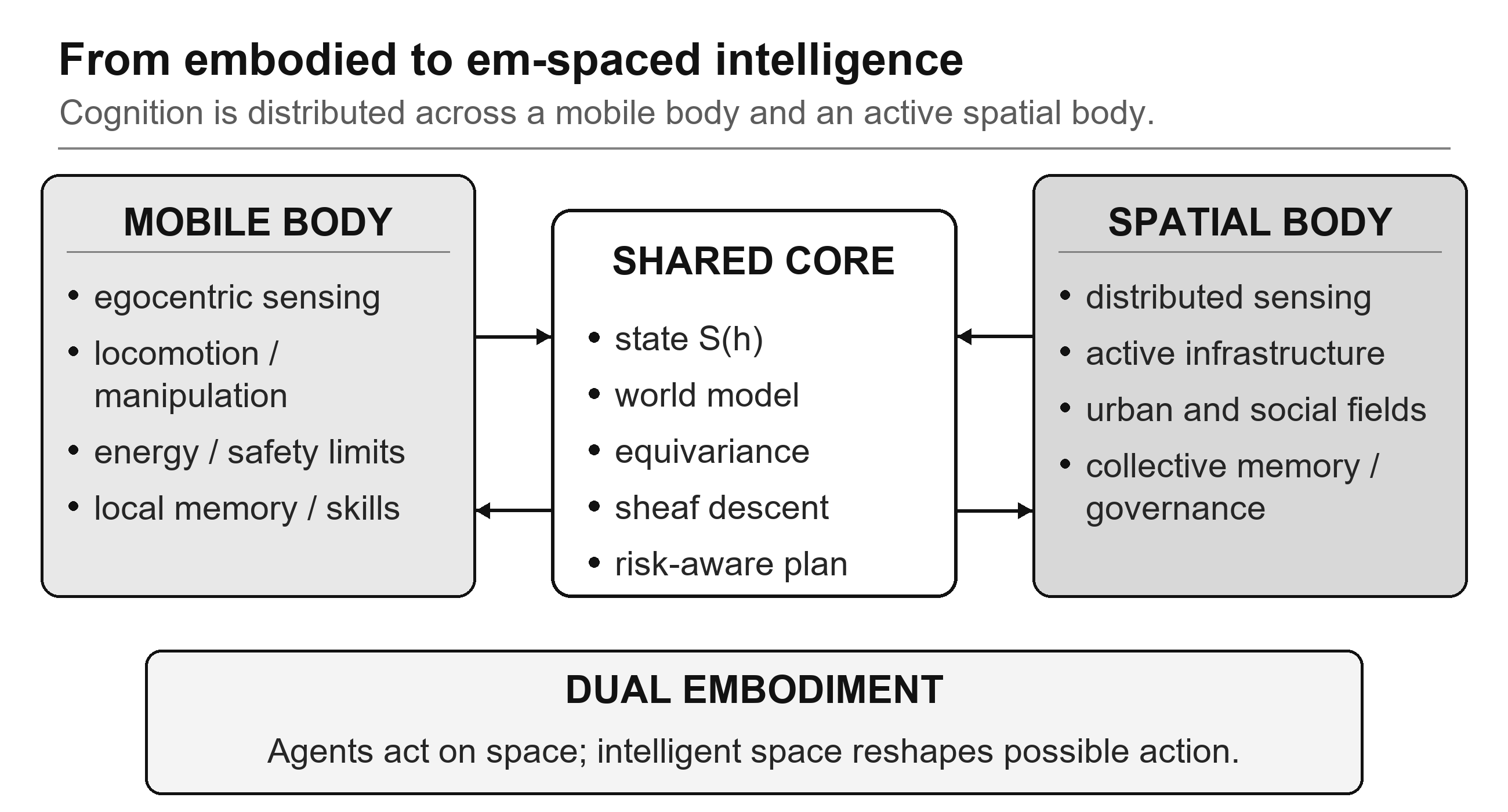}
\caption{Dual embodiment. The mobile body and the spatial body share a predictive core
but retain distinct sensors, actuators, memories and constraints.}
\label{fig:dual-embodiment}
\end{figure}

\subsection{Learning and control objective}
\label{sec:learning}

A trainable model can combine prediction with structural penalties rather than treating
geometry as a purely visual latent variable.
\begin{equation}
\mathcal{L} \;=\; \mathcal{L}_{\mathrm{pred}}
+ \lambda_{\mathrm{eq}}\Eeq
+ \lambda_{\mathrm{desc}}\Edesc
+ \lambda_{\mathrm{scale}}\Escale
+ \lambda_{\mathrm{cf}}\mathcal{L}_{\mathrm{cf}}
+ \lambda_{\mathrm{safe}}C_\Gamma .
\label{eq:loss}
\end{equation}
The first term scores future multimodal prediction. $\Eeq$ tests action
covariance, in the sense made systematic by the geometric deep learning programme
\citep{bronstein2021geometric};
$\Edesc$ tests local-to-global compatibility; $\Escale$ tests aggregation;
$\mathcal{L}_{\mathrm{cf}}$ scores counterfactual interventions; $C_\Gamma$ penalises
unsafe or institutionally inadmissible acts. A coupled planner then chooses both mobile
and environmental actions:
\begin{equation}
(\pi^{*},u^{*}) \;=\; \arg\min_{\pi,u}\ \Ex\Bigl[\ \sum_{\tau} c(S_\tau,\pi_\tau,u_\tau)
\ \Big|\ S_t \Bigr]
\quad\text{subject to}\quad (\pi_\tau,u_\tau)\in\Gamma(S_\tau).
\label{eq:planner}
\end{equation}
This is more than a robot using a smart building as a sensor. The building can change
illumination, ventilation, access, signage, signalling and information disclosure; the
robot can move, inspect and manipulate. Each intervention changes what the other can
learn and do. The object of intelligence is their coupled law.

\subsection{Experimental programme}
\label{sec:experimental-programme}

\paragraph{Experiment A --- cross-modal intervention prediction.} Instrument a building
with vision, acoustics, temperature, airflow, occupancy and door state. Hold out
complete intervention types, not random frames. Compare the full model with modality
concatenation, action-blind prediction and single-layer graph baselines.

\paragraph{Experiment B --- stratified transfer.} Train at room and building scales,
then test at floor, block and district scales. Report $\Escale$, held-out likelihood,
missing-sensor reconstruction and the stability of recovered topological features.

\paragraph{Experiment C --- dual embodiment.} Compare robot-only control,
environment-only automation and joint $(\pi,u)$ control on navigation, inspection,
emergency egress and human-robot coexistence. Report task success, energy, intervention
regret, safety violations and calibration under distribution shift.

\paragraph{Experiment D --- social and economic geometry.} Use road closure, timetable,
price or access-rule changes as explicit interventions. Test whether the inferred
reachability geometry predicts mobility and exchange beyond Euclidean distance while
remaining stable under alternative demographic or institutional partitions.

\paragraph{Experiment E --- context saturation.} Instrument a building with an ordered
family of modalities and add them one at a time, holding predictive capacity and
training budget fixed. Report the saturation residual $\Esat(c,c')$ of Corollary~8.6 as
a curve. The claim that the building has a well-defined spatial description relative to
this instrument set is the claim that the curve saturates; the modality at which it
stops falling identifies a sufficient context, and any later fall exhibits a modality
carrying genuinely new spatial information. As Section~\ref{sec:experiments} shows, the
noise-driven and refinement-driven components of the curve must be separated before
sufficiency can be asserted.

\paragraph{Experiment F --- holonomy.} Identify circuits in the street network whose
local frames do not compose to the identity, and record a vector-valued layer in local
frames along them. Proposition~4.2 predicts that a graph model with identity
restrictions shows residuals scaling with the circuit holonomy angle while the sheaf
model with Procrustes-estimated restrictions does not. This is the cleanest available
discriminator between the sheaf formalism and a multilayer network, because it concerns
a structural obstruction rather than a difference in fitting capacity.

\medskip\noindent
Spatial-intelligence benchmarks already reveal weaknesses in orientation, view selection
and interaction-aware 3D representation
\citep{hong2026esibench,wang2025site,zhu2025spa}. The proposed experiments add two
demands that current benchmarks rarely combine: compatibility across urban scales and
interventions performed by the environment itself.

\section{Euclidean, non-Euclidean and directional geometry}

Euclidean and non-Euclidean descriptions need not be rival ontologies. A building
survey may use Euclidean coordinates while acoustic travel time, wind transport, social
access and economic exchange define distinct effective geometries on the same
stratified base. Smoothly varying positive-definite tensors interpolate among
Riemannian geometries; Finsler costs express directional asymmetry; graphs and sheaves
handle discontinuities and singular junctions; hyperbolic embeddings accommodate
hierarchical reachability at low distortion \citep{nickel2017poincare}; metric-measure
and optimal-transport structures compare distributions and evolving mass
\citep{gromov2007metric,bao2000introduction,villani2009optimal,battiston2020networks,connes1994noncommutative}.
The configurational tradition in urban morphology has long made the same methodological
point in a different vocabulary, treating accessibility rather than metric distance as
the primary spatial variable
\citep{hillier1984social,hillier1996space,hillier2005art,batty2013new,yang2019value},
and recent work in that tradition has begun to couple it to generative and
manifold-based models of urban form
\citep{yang2022artificially,yang2025generative}. Continuity of description is
possible where the structure varies continuously, but topology-changing interventions
and strata transitions may be genuinely discontinuous.

Directional geometry deserves a brief note, stated conservatively. Some urban layers are
governed by directional rather than isotropic propagation: acoustic, electromagnetic and
radiative transfer concentrate along thin tubes, and the natural obstruction theory for
such layers is harmonic-analytic. Fefferman's ball-multiplier counterexample shows that
directional decompositions cannot be summed naively \citep{fefferman1971multiplier}, and
the recent resolution of the Kakeya set conjecture in three dimensions sharpens the
control of unions of tubes \citep{guth2026streamlined,wang2025volume}. We record the
connection as a layer-specific prior and nothing more. Kakeya estimates supply
inequalities and obstruction mechanisms; they do not determine a sensing geometry, do
not by themselves fix a minimum number of tomographic angles or beams, and do not supply
a reconstruction algorithm. Converting them into a stability guarantee requires a
specified forward operator, noise model and discretisation, which we do not develop here
and flag as future work rather than as a result of this paper.

\section{What, then, is space?}

Space is best defined neither as dimension alone nor as an arbitrary relation. Dimension
states how many independent degrees of freedom a chosen representation needs. Relation
states which distinctions and transitions matter. The latter is logically prior: without
adjacency, reachability, compatibility and admissible transformation, a dimension is a
number without spatial content. Yet not every relation is spatial. A relation becomes
spatial when it is localisable, composable, testable by intervention, stable enough to
support counterfactual prediction and compatible across overlapping probes.

\begin{defn}{8.1}{space, provisional form}
For a specified family of probes, modalities, actions and scales, space is the minimal
local-to-global relational object that makes the intervention-conditioned
transformations of those modalities jointly representable and predictively sufficient.
\end{defn}

Definition~8.1 invites an obvious objection, and the objection must be met rather than
deflected. If space is defined relative to a specified family of probes, modalities and
actions, then changing the family changes the space, and the word \emph{invariant} in
the title of this paper is doing no work: what has been described is an
observer-dependent artefact. The answer is that the family-indexed constructions are not
independent of one another. They are functorially related, and the object that
Definition~8.1 names is properly the limit of the whole system rather than any one of
its terms.

\subsection{Observational contexts and refinement}

\begin{defn}{8.2}{observational context}
An observational context is a tuple $c=(B_c,A_c,\Pi_c,K_c)$ consisting of a subfamily of
probes, a subset of modalities, a countable class of admissible policies and a set of
horizons. Write $S_c$ for the predictive state \eqref{eq:predictive-state} formed from
the coordinates indexed by $\Pi_c\times K_c$ using the modalities in $A_c$ at the probes
in $B_c$, and $\sim_c$ for the induced equivalence on histories. Order contexts by
componentwise inclusion: $c\preccurlyeq c'$ when $B_c\subseteq B_{c'}$,
$A_c\subseteq A_{c'}$, $\Pi_c\subseteq\Pi_{c'}$ and $K_c\subseteq K_{c'}$. The poset
$(\Ctx,\preccurlyeq)$ is directed, since the componentwise union of two contexts is
again a context and countability is preserved by finite unions.
\end{defn}

\begin{lem}{8.3}{refinement}
If $c\preccurlyeq c'$ then ${\sim_{c'}}\subseteq{\sim_c}$, and there is a unique
surjection $p_{c'c}\colon S_{c'}(\Hist)\to S_c(\Hist)$ with $S_c = p_{c'c}\circ S_{c'}$.
These maps satisfy $p_{cc}=\mathrm{id}$ and $p_{c'c}\circ p_{c''c'} = p_{c''c}$ whenever
$c\preccurlyeq c'\preccurlyeq c''$.
\end{lem}

\begin{proof}
The coordinates of $S_c$ form a subset of the coordinates of $S_{c'}$, so $S_c$ is the
corresponding coordinate projection of $S_{c'}$; in particular $S_{c'}(h)=S_{c'}(h')$
implies $S_c(h)=S_c(h')$, which is both the inclusion of equivalence relations and the
well-definedness of $p_{c'c}$ on the image. Uniqueness holds because $p_{c'c}$ is
determined on $S_{c'}(\Hist)$, which is the whole of its domain, and the cocycle
identities are the corresponding identities for coordinate projections.
\end{proof}

\subsection{Space as a projective limit}

\begin{thm}{8.4}{space as a projective limit}
The pairs $(\{S_c(\Hist)\},\{p_{c'c}\})$ form a projective system over the directed
poset $(\Ctx,\preccurlyeq)$. Let
\begin{equation}
S_\infty \;:=\; \varprojlim_{c\in\Ctx} S_c(\Hist)
\;=\; \Bigl\{ (s_c)_c \in \prod_c S_c(\Hist)\ :\
p_{c'c}(s_{c'}) = s_c \ \text{ whenever } c\preccurlyeq c' \Bigr\},
\label{eq:projective-limit}
\end{equation}
and let $S_{\mathrm{can}}(h) := (S_c(h))_c$. Then:
\begin{enumerate}[label=(\roman*),leftmargin=2.2em]
\item $S_{\mathrm{can}}$ takes values in $S_\infty$ and $S_c = \mathrm{pr}_c\circ S_{\mathrm{can}}$ for every $c$;
\item $\sigma(S_{\mathrm{can}}) = \bigvee_c \sigma(S_c)$, so $S_{\mathrm{can}}$ is the
minimal statistic sufficient for all contexts simultaneously;
\item each $S_c$ is an interventional invariant relative to $c$, whereas
$S_{\mathrm{can}}$ is invariant under change of context;
\item if there is a context $c_0$ such that $p_{cc_0}$ is injective for every
$c\succcurlyeq c_0$, then $S_\infty\cong S_{c_0}$, and we call $c_0$ a
\emph{sufficient context};
\item if $\Ctx$ admits a countable cofinal chain
$c_1\preccurlyeq c_2\preccurlyeq\cdots$ with each $S_{c_n}(\Hist)$ standard Borel and
each connecting map Borel, then $S_\infty$ is a Borel subset of a Polish product and is
therefore standard Borel.
\end{enumerate}
\end{thm}

\begin{proof}
The system is projective by Lemma~8.3 and directed by Definition~8.2. (i)~The
compatibility conditions defining $S_\infty$ are exactly the identities
$S_c = p_{c'c}\circ S_{c'}$, which hold pointwise. (ii)~Each $S_c$ factors through
$S_{\mathrm{can}}$ by (i), so $\bigvee_c\sigma(S_c)\subseteq\sigma(S_{\mathrm{can}})$;
conversely $S_{\mathrm{can}}$ is measurable with respect to the product
$\sigma$-algebra generated by the coordinates $S_c$, giving the reverse inclusion.
Minimality then follows from Theorem~3.2(ii) applied coordinatewise. (iii)~Immediate
from (i) and the definition of the limit, which involves no choice of context. (iv)~If
$p_{cc_0}$ is injective for all $c\succcurlyeq c_0$, then a compatible family is
determined by its $c_0$-component, because every $c$ is dominated by some
$c'\succcurlyeq c_0$ by directedness; the map $S_\infty\to S_{c_0}$ is thus a bijection
with inverse given by transport along the system. (v)~A cofinal countable chain
determines the limit, which is then the set of compatible sequences in a countable
product of Polish spaces --- a Borel subset, since it is a countable intersection of
preimages of diagonals under Borel maps, and Borel subsets of standard Borel spaces are
standard Borel.
\end{proof}

\begin{defn}{8.5}{space, final form}
For a class $\Ctx$ of admissible observational contexts, space is the projective limit
$S_\infty = \varprojlim_{c\in\Ctx} S_c(\Hist)$ of the context-indexed predictive
quotients, together with the action groupoid $G$ and the descent data that make each
$S_c$ a compatible local description.
\end{defn}

\begin{rem}{8.5a}{the relativism objection answered}
Every individual $S_c$ is context-relative, exactly as every coordinate chart is
chart-relative. What Theorem~8.4(iii) supplies is that the family is functorial and the
limit is not context-relative; what varies between observers is which finite
approximation to $S_\infty$ their instruments give them access to, and that is a
statement about epistemic access rather than about the object. Clause~(v) says the limit
is a legitimate measurable space, so the definition is not merely formal. Clause~(iv)
says that in favourable cases the limit is attained at a finite stage, so a finitely
instrumented city can in principle carry a complete spatial description rather than an
endlessly improvable approximation. This is the sense in which the construction is an
invariant, and it is the sense the title intends.
\end{rem}

\begin{cor}{8.6}{saturation is falsifiable}
For $c\preccurlyeq c'$ define the saturation residual
\begin{equation}
\begin{split}
\Esat(c,c') \;:=\;\ & \sup_{\pi,k}\ \Ex\, D\bigl(P^{\pi,k}(\cdot\mid \Hist)\,\big\|\,Q^{\pi,k}(\cdot\mid S_c(\Hist))\bigr)\\[2pt]
&-\ \sup_{\pi,k}\ \Ex\, D\bigl(P^{\pi,k}(\cdot\mid \Hist)\,\big\|\,Q^{\pi,k}(\cdot\mid S_{c'}(\Hist))\bigr).
\end{split}
\label{eq:esat}
\end{equation}
The assertion that $c_0$ is a sufficient context predicts $\Esat(c_0,c')=0$ within
tolerance for every $c'\succcurlyeq c_0$. A nonzero residual at any $c'$ falsifies
sufficiency and exhibits a modality, probe or policy carrying spatial information not
present in $c_0$.
\end{cor}

In the synthetic system of Section~\ref{sec:experiments} the residual falls by
$3.7\times10^{-1}$ from $\{v\}$ to $\{v,a\}$ and thereafter, in the noiseless control,
to below $10^{-29}$ with no further decrease, so $c_0=\{v,a\}$ is a sufficient context
there. Section~\ref{sec:experimental-programme} states the corresponding urban protocol
as Experiment~E. It is worth being explicit that saturation is a substantive empirical
claim about a city and not a theorem about cities: a city whose saturation curve never
flattens would be one for which no finite instrument set determines a spatial
description, and nothing in this framework rules that out.

\subsection{Time}

Time enters through composition and irreversibility. If $\Phi_{s,t}$ maps states from
$s$ to $t$, temporal coherence requires
\begin{equation}
\Phi_{t,u}\circ\Phi_{s,t} = \Phi_{s,u},
\qquad
\Phi_{t,t} = \mathrm{id}.
\label{eq:time}
\end{equation}
This algebra supplies order and duration only when coupled to clocks, causal cones,
dissipation or action cost. Space and time are therefore not produced merely by
extending a list of dimensions. Space records the compatibility of possible
co-existence and transition; time records the composable ordering, rate and
irreversibility of change.

\section{Limitations and boundary conditions}
\label{sec:limitations}

\paragraph{Identifiability.} Theorem~3.7 recovers the latent space only up to the
centraliser of the intervention group, and only under joint separation, equivariance and
interventional faithfulness. Each hypothesis is a real risk in deployment. The policy
family must excite the relevant degrees of freedom, which by Lemma~3.5 is a positivity
condition on the data-generating regime rather than a property of the city; modalities
must jointly separate the states of interest, which no single instrument suite
guarantees; and equivariance of the learned encoder is enforced by penalty and verified
as a residual, not derived.

\paragraph{Confounding.} Assumption~I2 is the strongest substantive assumption in the
paper. In an instrumented building, interventions are often selected by facility
managers, schedules or occupants in response to conditions that the sensor suite does
not record, which is precisely unobserved confounding. Where it fails, the estimated
predictive state is a statistical rather than an interventional object, and
Corollary~3.8 does not apply.

\paragraph{Nonstationarity.} Urban institutions and populations change the observation
and intervention laws. A fixed sheaf can become obsolete; restriction maps and strata
may need online change-point tests. In the language of Section~8 the context class
itself drifts, so the projective system is indexed by a moving poset and the limit of
Theorem~8.4 need not exist.

\paragraph{Normativity.} Social and economic geometries contain power, exclusion and
value. Predictability does not make a policy legitimate. Governance constraints must be
model inputs, not conclusions deduced from geometric efficiency. This is sharpened by
Lemma~3.5: because only executed interventions enter the geometry, whoever controls
which interventions are performed thereby controls which spatial distinctions the model
is able to represent at all.

\paragraph{Computation.} Exact inference over rich sheaves, long horizons and joint
robot-environment interventions is generally intractable. Approximation changes the
effective predictive quotient and must be reported as part of the model.

\paragraph{Evidence.} The synthetic tests check algebraic coherence and the specific
claims of Propositions~4.2, 4.3 and Theorem~8.4(iv) in a system where the ground truth
is known. They supply no field validation. A publishable empirical claim requires
preregistered interventions, real multimodal data, ablations and external replication.

\section{Conclusion}

The central claim can be said plainly. Vision, hearing, touch and the many fields and
flows of a city do not disclose space because they look alike. They disclose it because
the same actions produce stable, mutually constraining and counterfactually testable
changes in them. Mathematical space supplies the formal object, physical space supplies
empirical law, and perceptual space supplies the minimal predictive quotient needed for
action. Their unity lies in an interventional invariant, not in the erasure of their
distinctions.

A stratified base, typed fibres and sheaf descent then provide a disciplined way to
place buildings, fields, mobility, institutions and exchange in one model. The same
construction clarifies em-spaced intelligence: an intelligent agent does not merely
occupy a passive environment; agent and spatial body jointly determine the future
possibilities of sensing and action. The mathematical programme is consequently
falsifiable. Its predictions fail when modalities do not align under intervention, local
models do not glue, scale maps do not commute within tolerance, or a simpler predictive
state performs equally well. Those failure conditions are not weaknesses of the
definition. They are what make it scientific.

\appendix

\section{Notation}

\begingroup
\small
\begin{longtable}{L{2.9cm} L{9.9cm}}
\caption{Principal notation.}\label{tab:notation}\\
\toprule
\textbf{Symbol} & \textbf{Meaning}\\
\midrule
\endfirsthead
\toprule
\textbf{Symbol} & \textbf{Meaning}\\
\midrule
\endhead
\bottomrule
\endfoot
$B$ & Site of probes or stratified urban base.\\
$\Shf$ & Sheaf of typed local states; $\Shf(U)$ is the state space over probe $U$.\\
$G$ & Action groupoid; arrows are executable movements or interventions.\\
$h_\alpha$ & Observation map for modality $\alpha$.\\
$\Phi_g$ & State transformation induced by action $g$.\\
$T^g_\alpha$ & Corresponding transformation in modality $\alpha$.\\
$\Hist$ & Space of multimodal histories.\\
$\Pi$, $K$ & Admissible policies and prediction horizons.\\
$S(h)$ & Canonical predictive state, a product of conditional future laws.\\
$\delta_\Shf$, $L_\Shf$ & Sheaf coboundary and weighted sheaf Laplacian.\\
$R_{\ell\to m}$ & Cross-scale restriction or coarsening map.\\
$c$, $\Ctx$ & Observational context and the directed poset of contexts.\\
$S_c$, $p_{c'c}$ & Context-relative predictive quotient and the refinement projections.\\
$S_\infty$ & Projective limit of the context-indexed quotients; space in the final sense.\\
$Z_{\Homeo(X)}(\Phi(G))$ & Centraliser of the intervention action; the residual ambiguity group.\\
$\Hol(\Shf,v_0)$ & Holonomy group of the cellular sheaf $\Shf$ at base vertex $v_0$.\\
$\eta(\tau)$, $\tau_\epsilon$ & Interventional separation modulus and the resolution it implies.\\
$b$, I1--I3 & Behaviour policy and the causal assumptions making $P$ interventional.\\
\end{longtable}
\endgroup

\section{Acceptance criteria for an empirical paper}

\paragraph{B1. Intervention coverage.} The training and test sets must state which
actions are observed, which are held out, and whether policy support is sufficient to
distinguish the claimed spatial states.

\paragraph{B2. Competing geometries.} Compare Euclidean coordinates, graph distance,
learned latent distance, action cost and the sheaf model. Report where each succeeds
rather than selecting one post hoc.

\paragraph{B3. Structural ablations.} Remove action conditioning, cross-modal alignment,
nontrivial restriction maps, cross-scale penalties and environmental actuation one at a
time.

\paragraph{B4. Calibration and shift.} Report conditional likelihood or proper scoring
rules, uncertainty calibration and performance under unseen buildings, populations,
weather and institutional regimes.

\paragraph{B5. Governance.} State who controls spatial interventions, whose costs enter
the objective and which safety or access constraints are inviolable.

\bibliographystyle{plainnat}
\bibliography{references}

\end{document}